\documentclass[preprint,3p]{elsarticle}
\usepackage{amssymb}
\usepackage[caption=false]{subfig}
\usepackage{graphicx}
\usepackage{supertabular,booktabs}
\usepackage{amsmath,amssymb}
\usepackage{algorithm,algorithmic}
\usepackage{multirow}
\usepackage{multicol}
\usepackage{makecell}
\usepackage{diagbox}
\usepackage{hyperref}
\usepackage{color}

\newenvironment{proof}{{\noindent\bf Proof}\quad}{\hfill $\square$\par}

\newtheorem{lemma}{Lemma}
\newtheorem{theorem}{Theorem}
\newtheorem{definition}{Definition}
\newtheorem{example}{Example}

\newtheorem{problem}{Problem}

\journal{Journal of \LaTeX\ Templates}

\begin{document}

\begin{frontmatter}

\title{Influence of Binomial Crossover on Approximation Error of Evolutionary Algorithms}



\author[WUT]{Cong Wang}
\ead{1291092964@qq.com}

\author[TU]{Jun He}
\ead{jun.he@ntu.ac.uk}

\author[WUT]{Yu Chen\corref{mycorrespondingauthor}}
\cortext[mycorrespondingauthor]{Corresponding author.}
\ead{ychen@whut.edu.cn}

\author[WHU1,WHU2]{Xiufen Zou}
\ead{xfzou@whu.edu.cn}
\address[WUT]{School of Science, Wuhan University of Technology, Wuhan, 430070, China}
\address[TU]{Department of Computer Science, Nottingham Trent University, Clifton Campus, Nottingham, NG11 8NS, UK}
\address[WHU1]{School of Mathematics and Statistics, Wuhan University, Wuhan, 430072, China}
\address[WHU2]{Computational Science Hubei Key Laboratory, Wuhan University, Wuhan, 430072, China}

\begin{abstract}
Although differential evolution (DE) algorithms perform well on a large variety of complicated optimization problems, only a few theoretical studies are focused on the working principal of DE algorithms. To make the first attempt to reveal the function of binomial crossover, this paper aims to answer whether it can reduce the approximation error of evolutionary algorithms. By investigating the expected approximation error and the probability of not finding the optimum, we conduct a case study of comparing two evolutionary algorithms with and without binomial crossover on two classical benchmark problems: OneMax and Deceptive. It is proven that using binomial crossover leads to dominance of transition matrices. As a result, the algorithm with binomial crossover asymptotically outperforms that without crossover on both OneMax and Deceptive, and outperforms on OneMax, however, not on Deceptive. Furthermore, an adaptive parameter strategy is proposed which can strengthen superiority of binomial crossover on Deceptive.
\end{abstract}


\end{frontmatter}

\section{Introduction}

Evolutionary algorithms (EAs) are a  big family of randomized search heuristics  inspired from biological evolution. 
Many evolutionary algorithms (EAs) such as genetic algorithm (GA) and evolutionary strategies (ES) use a crossover operator besides mutation and selection. Many empirical studies demonstrate that crossover which combines genes of two parents to generate new offspring is helpful to convergence of   EAs. Theoretical results on runtime analysis validate the promising function of crossover in EAs~\citep{jansen2002analysis,kotzing2011crossover,corus2017standard,dang2017escaping,sudholt2017crossover,pinto2018simple, oliveto2020tight,lengler2020large,lehre2008crossover,doerr2012crossover,doerr2013more,sutton2021fixed}, whereas there are also some cases that crossover cannot be helpful~\citep{richter2008ignoble,antipov2021effect}.


Differential evolution (DE) algorithms implement crossover operations in a different way. By exchanging components of target vectors with donor vectors, continuous DE algorithms achieve competitive performance on a large variety of complicated problems~\citep{Das2011,DAS20161,SEPESYMAUCEC2019100428,BILAL2020103479}, and its competitiveness is to great extent attributed to the employed crossover operations~\citep{Lin2011}.
Besides the theoretical studies on continuous DE \citep{OPARA2019546} as well as the runtime analysis to reveal working principle of BDE~\citep{DOERR2020110}, there were no theoretical results reported on how crossover influences performance of discrete-coded DE algorithms.

Nevertheless, there is a gap between runtime analysis and practice. Since EAs belong to randomized search, their optimization time to reach an optimum is uncertain and could be even infinite in continuous optimization~\citep{CHEN2021200}. Due to this reason, optimization time is seldom used for evaluating the performance of EAs in computer simulation.  While  EAs stop after running finite generations, their performance is evaluated  by  solution quality such as the  mean and median of the fitness value or approximation error~\citep{xu2020helper}. In theory,  solution quality can be measured for given iteration budget by  the expected fitness value~\citep{JANSEN201439} or  approximation error~\citep{he2016analytic,WANG2021}, which contributes to the analysis framework named as fixed budget analysis. A fixed budget analysis on immune-inspired hypermutations leads to theoretical results that are very different from those of runtime analysis but consistent to the empirical results, which further demonstrates that the perspective of fixed budget computations
provides valuable information and additional insights for performance of randomized search heuristics~\citep{Jasen2014TEVC}.

 In this paper, solution quality of an EA after running finite generations is measured by two metrics: the expected value of the approximation error and the error tail probability.   The former measures  solution quality, that is the fitness gap  between a solution and optimum. The latter is the probability distribution of the error over error levels, which measures the probability of finding the optimum.
 An EA is said to outperform   another if for the former EA, its error and tail probability are smaller. Furthermore, an EA is said to asymptotically outperform   another if for the former EA, its error and tail probability are smaller after a sufficiently large number of generations.

The research question of this paper is whether the binomial crossover operator can help reduce the approximation error. As a pioneering work on this topic, we investigate a $(1+1)EA_C$ that performs the binomial crossover on an individual and an offspring generated by mutation, and compare a $(1+1) EA$ without crossover and its variant  $(1+1)EA_C$ on two classical problems, OneMax and Deceptive. By splitting the objective space into error levels, the analysis is performed based on the Markov chain models~\citep{he2003towards,He2016}. Given  the two EAs, the comparison of their performance are drawn from the comparison of their transition  probabilities, which are estimated by investigating the bits preferred to by evolutionary operations. Under some conditions, $(1+1)EA_C$ with binomial crossover   outperforms $(1+1) EA$ on OneMax, but not on Deceptive;  however, by adding an adaptive parameter mechanism arising from theoretical results, $(1+1)EA_C$ with binomial crossover outperforms  $(1+1) EA$   on Deceptive too.

The rest of this paper is organized as follows. Section \ref{SecRelWork} reviews related theoretical work.  Preliminary contents for our theoretical analysis are presented in Section \ref{SecPre}. Then, the influence of the binomial crossover on transition probabilities is investigated in Section \ref{SecTrans}. Section \ref{SecAsy} conducts an analysis of the asymptotic performance of EAs. To reveal how  binomial crossover works on the performance of EAs for consecutive iterations, the OneMax problem and the Deceptive problem are investigated in Sections \ref{SecExpi} and \ref{SecExpr}, respectively. Finally, Section \ref{SecCon} presents the conclusions and discussions.

\section{Related Work}\label{SecRelWork}

\subsection{Theoretical Analysis of Crossover}

To understand how  crossover influences the performance of EAs, Jansen \emph{et al.} \cite{jansen2002analysis} proved that an EA using crossover can reduce the expected optimization time from super-polynomial to a polynomial of small degree on the function Jump. K\"{o}tzing \emph{et al.} 
\cite{kotzing2011crossover} investigated   crossover-based EAs on the  functions OneMax and Jump and showed  the potential speedup by crossover when combined with a fitness-invariant bit shuffling operator in terms of optimization time.  For a simple GA without shuffling, they found that the crossover probability has a drastic impact on the performance on Jump. Corus and Oliveto
\cite{corus2017standard}   rigorously obtained an upper bound
on the runtime of standard steady state GAs to hillclimb the OneMax function and proved that the steady-state
EAs are 25\% faster than their mutation-only counterparts. Their analysis also suggests
that larger populations may be faster than populations of size 2. Dang \emph{et al.} \cite{dang2017escaping} revealed that the interplay between crossover and mutation may result in a sudden burst of diversity on the Jump test function and reduce the expected optimization time  compared to mutation-only algorithms like the (1+1) EA.
For royal road functions and OneMax, Sudholt \cite{sudholt2017crossover} analyzed uniform crossover and k-point crossover and proved that crossover makes every $(\mu+\lambda)$ EA at least twice as fast as the fastest EA using only standard bit mutation. Pinto and Doerr
\cite{pinto2018simple} provided a simple proof of a crossover-based genetic algorithm (GA) outperforming any mutation-based black-box heuristic on the classic benchmark OneMax. Oliveto  \emph{et al.} \cite{oliveto2020tight} obtained a tight lower bound on the expected runtime of the (2 + 1) GA on OneMax. Lengler and Meier
\cite{lengler2020large} studied the positive effect of using larger population sizes and crossover on Dynamic BinVal.

Besides   artificial benchmark functions, several theoretical studies also show that crossover is   helpful on  non-artificial problems.
Lehre and Yao \cite{lehre2008crossover}  proved that the use of crossover in the $(\mu+1)$ Steady State Genetic Algorithm may reduce the runtime from exponential to polynomial for some instance classes of the problem of computing unique input–output (UIO) sequences. Doerr  \emph{et al.} 
\cite{doerr2012crossover,doerr2013more}  analyzed  EAs on  the all-pairs shortest path problem. Their results confirmed that the EA with a crossover operator is significantly faster  than without in  terms of the expected optimization time. Sutton 
\cite{sutton2021fixed} investigated the effect of crossover on the closest string problem and proved that a multi-start $(\mu+1)$ GA  required less randomized fixed-parameter tractable (FPT) time than that with disabled crossover.


However, there is some evidence that crossover is not always helpful. Richter  \emph{et al.} \cite{richter2008ignoble}   constructed Ignoble Trail functions and proved that  mutation-based EAs optimize them more efficiently than GAs with crossover. The later need exponential optimization time. Antipov and Naumov 
\cite{antipov2021effect}
compared crossover-based algorithms on   RealJump functions with a slightly shifted optimum. The runtime of all considered algorithms  increases on  RealJump. The hybrid GA fails to find the shifted optimum with high probability.

\subsection{Theoretical Analysis of Differential Evolution Algorithms}

Although numerical investigations of DEs have been widely conducted, only a few theoretical studies paid attention to theoretical analysis, most of which are focused on continuous DEs~\citep{OPARA2019546}. By estimating the probability density function of generated individuals, Zhou  \emph{et al.} \cite{Zhou2016Analysis} demonstrated that the selection mechanism of DE, which chooses mutually different parents for generation of donor vectors, sometimes does not work positively on performance of DE.  Zaharie and Micota \cite{ZAHARIE20091126,zaharie2008statistical,Zaharie2017} investigated influence of the crossover rate on both the distribution of the number of mutated components
and the probability for a component to be taken from the mutant vector, as well as the influence of mutation and crossover on the diversity of intermediate population. Wang and Huang \cite{WANG20103263} attributed the DE to a one-dimensional stochastic model, and investigated how the probability distribution of population is connected to the mutation, selection and crossover operations of DE. Opara and Arabas \cite{OPARA201853} compared several variants of differential mutation using characteristics of their expected mutants’ distribution, which demonstrated that the classic mutation operators yield similar
search directions and differ primarily by the mutation range. Furthermore, they formalized the contour fitting notion and derived an analytical model that links the differential mutation operator with the adaptation of
the range and direction of search~\citep{OPARA2019100441}.

By investigating expected runtime of the binary differential evolution (BDE) proposed by Gong and Tuson \cite{Gong2007}, Doerr and Zhang \cite{DOERR2020110} performed a first fundamental analysis on the working principles of discrete-coded DE. It was shown that BDE optimizes the important decision variables, but is hard to find the optima for decision variables with small influence on the objective function. Since BDE generates trial vectors by implementing a binary variant of binomial crossover accompanied by the mutation operation, it has characteristics significantly different from classic EAs or estimation-of-distribution algorithms.

\subsection{Fixed Budget Analysis and Approximation Error}

To bridge the wide gap between theory and application, Jasen and Zarges \cite{JANSEN201439} proposed a fixed budget analysis (FBA) framework of randomized search heuristics (RSH), by which the fitness of random local search and (1+1)EA were investigated for given iteration budgets. Under the framework of FBA, they analyzed the any time performance of EAs and artificial immune systems on a proposed dynamic benchmark problem~\citep{Jansen2014}. Nallaperuma  \emph{et al.} \cite{Nallaperuma2017} considered the well-known traveling salesperson problem (TSP) and  derived the lower bounds of the expected fitness gain for a specified number of generations. Doerr  \emph{et al.} \cite{Doerr2013} built a bridge between runtime analysis and FBA, by which a huge body of work and a large collection of tools for the analysis of the expected optimization time could meet the new challenges introduced by the new fixed budget perspective. Based on the Markov chain model of RSH, Wang  \emph{et al.} \cite{WANG2021} constructed a general framework of FBA, by which they got the analytic expression of approximation error instead of asymptotic results of expected fitness values.

Considering that runtime analysis demonstrated that  hypermutations tend to be inferior
on typical example functions, Jansen and Zarges \cite{Jasen2014TEVC} conducted an FBA to explain why artificial immune systems are popular in spite
of these proven drawbacks. Although the single point mutation in random local search (RLS) outperforms than the inversely fitness-proportional mutation (IFPM)  and the somatic contiguous hypermutation (CHM) on OneMax in terms of expected optimization time, it was shown that  IFPM and CHM could be better while FBA is performed by considering  different starting points and varied iteration budgets. The results show that the traditional perspective of expected optimization time may be unable to
explain observed good performance that is due to limiting the length of runs.  Therefore, the perspective of fixed budget computations
provides valuable information and additional insights.

\section{Preliminaries}\label{SecPre}
\subsection{Problems}
Consider a maximization problem
\begin{equation*}
  \max f(\mathbf{x}),\quad\mathbf{x}=(x_1,\dots,x_n)\in\{0,1\}^n,
\end{equation*}
Denote its optimal solution by $\mathbf{x}^*$ and optimal objective value by  $f^*$.
The quality of a solution $\mathbf x$ is  evaluated by its approximation error $e(\mathbf x):=\mid f(\mathbf x)-f^*\mid$. The error $e(\mathbf x)$ takes finite values, called error levels:
\begin{equation*}
  e(\mathbf x)\in\{e_0,e_1,\dots,e_L\},\quad 0= e_0\le e_1\le\dots\le e_L,
\end{equation*}
where $L$ is an non-negative integer. $\mathbf{x}$ is called \emph{at the level $i$} if $e(\mathbf{x})=e_i$, $i\in\{0,1,\dots,L\}$. The collection of solutions at level $i$ is denoted by $\mathcal{X}_i$.

Two instances, the uni-modal OneMax problem and the multi-modal Deceptive problem, are considered in this paper.
\begin{problem}\label{OneM}(\textbf{OneMax})
\begin{equation*}
  \max f(\mathbf x)=\sum_{i=1}^nx_i,
\end{equation*}
where $\mathbf x=(x_1,\dots,x_n)\in \{0,1\}^n$.
\end{problem}
\begin{problem}\label{Dec}(\textbf{Deceptive})
\begin{equation*}
  \max f(\mathbf x)=\left\{\begin{aligned}& \sum_{i=1}^nx_i, && \mbox{if }\sum_{i=1}^nx_i>n-1,\\ &  n-1-\sum_{i=1}^nx_i, && \mbox{otherwise.} \end{aligned}\right.
\end{equation*}
where $\mathbf x=(x_1,\dots,x_n)\in \{0,1\}^n$.
\end{problem}
Both  OneMax and  Deceptive  can be represented in the form
\begin{equation}\label{OP1}
  \max f(\mid\mathbf{x}\mid),
\end{equation}
where $\mid\mathbf{x}\mid:=\sum_{i=1}^{n}x_i$. Error levels of (\ref{OP1}) take only $n+1$ values.

For the OneMax problem, both exploration and exploitation are helpful to convergence of EAs to the optimum, because exploration accelerates the convergence process and exploitation refines the precision of approximation solutions. However, for the Deceptive problem, local exploitation leads to convergence to the local optimum, but it in turn increases the difficulty to jump to the global optimum. That is, exploitation hinders convergence to the global optimum of the Deceptive problem, thus, the performance  of EAs are dominantly influenced by their exploration ability.

\subsection{Evolutionary Algorithms}

  \begin{algorithm}[ht]
  \caption{$(1+1)EA$}\label{Alg1}
  \begin{algorithmic}[1]
  \STATE counter $t=0$;
  \STATE randomly generate a solution $\mathbf x_0=(x_1,\dots,x_n)$;
  \WHILE{the stopping criterion is not satisfied}
  \FOR{$i=1,2,\dots,n$}
  \STATE   \begin{align}\label{Mut1}y_i=\left\{\begin{aligned}& 1-x_i, && \mbox{if } rnd_i<p_m,\\ & x_i, && \mbox{otherwise},\end{aligned}\right. \quad rnd_i\sim U[0,1];
  \end{align}
  \ENDFOR
  \IF{$f(\mathbf{y})\ge f(\mathbf{x}_t)$}
  \STATE $\mathbf{x}_{t+1}=\mathbf{y}$;
  \ENDIF
  \STATE $t=t+1$;
  \ENDWHILE
  \end{algorithmic}
  \end{algorithm}


For the sake of analysis on binomial crossover excluding influence of population, $(1+1)EA$ (presented by Algorithm \ref{Alg1}) without crossover is taken  as the baseline algorithm in our study. Its candidate solutions are generated by the bitwise mutation with probability $p_m$. Binomial crossover is added to $(1+1)EA$,  getting  $(1+1)EA_C$ which is illustrated in Algorithm \ref{Alg3}. The $(1+1)EA_{MC}$ first performs  bitwise mutation with probability $q_m$,  and then applies binomial crossover with rate $C_R$ to generate a candidate solution for selection.

  \begin{algorithm}[ht]
  \caption{$(1+1)EA_C$}\label{Alg3}
  \begin{algorithmic}[1]

  \STATE counter $t=0$;
  \STATE randomly generate a solution $\mathbf x_0=(x_1,\dots,x_n)$;
  \WHILE{the stopping criterion is not satisfied}
  \STATE set $rndi\sim U\{1,2,\dots,n\}$;
  \FOR{$i=1,2,\dots,n$}
  \STATE  \begin{equation}\label{Mut2}v_i=\left\{\begin{aligned}& 1-x_i, && \mbox{if }  rnd1_i<q_m,\\ & x_i, && \mbox{otherwise},\end{aligned}\right.\quad rnd1_i\sim U[0,1];\end{equation}

  \STATE  \begin{equation}\label{Cross}y_i=\left\{\begin{aligned}& v_i, && \mbox{if } i=rndi \mbox{ or } rnd2_i<C_R,\\ & x_i, && \mbox{otherwise},\end{aligned}\right.\quad rnd2_i\sim U[0,1]\end{equation}

  \ENDFOR
  \IF{$f(\mathbf{y})\ge f(\mathbf{x}_t)$}
  \STATE $\mathbf{x}_{t+1}=\mathbf{y}$;
  \ENDIF
  \STATE $t=t+1$;
  \ENDWHILE
  \end{algorithmic}
  \end{algorithm}

The EAs investigated in this paper can be modeled as homogeneous Markov chains~\citep{he2003towards,He2016}. Given the error vector
\begin{equation}\label{ErrVec}
  \mathbf{\tilde{e}}=(e_0,e_1,\dots,e_L)',
\end{equation}
and the initial distribution
\begin{equation}\label{IniDis}
  \mathbf{\tilde{q}}^{[0]}=(q_0^{[0]},q_1^{[0]},\dots,q_L^{[0]})'
\end{equation}
 the transition matrix of  $(1+1)EA$ and $(1+1)EA_C$ for the optimization problem (\ref{OP1}) can be written  in the form
\begin{equation}\label{Trans}
 \mathbf{\tilde{R}}=(r_{i,j})_{(L+1)\times (L+1)},
\end{equation}
where
\begin{equation*}
  r_{i,j}=\Pr\{\mathbf{x}_{t+1}\in\mathcal{X}_i\mid\mathbf{x}_{t}\in\mathcal{X}_j\},\quad i,j=0,\dots,L.
\end{equation*}
Recalling that the solutions are updated by the elitist selection, we know   $\mathbf{\tilde{R}}$ is an upper triangular matrix that can be partitioned as
\begin{equation*}
  \mathbf{\tilde{R}}=\left(
               \begin{array}{cc}
                 1 & \mathbf{r}_0 \\
                 \mathbf{0} & \mathbf{R} \\
               \end{array}
             \right),
\end{equation*}
where $\mathbf{R}$ is the transition submatrix depicting the transitions between non-optimal states.

\subsection{Transition Matrices}
By elitist selection, a candidate $\mathbf{y}$ replaces a solution $\mathbf{x}$ if and only if $f(\mathbf{y})\ge f(\mathbf{x})$,  which is achieved if ``$l$ preferred  bits'' of $\mathbf{x}$ are changed. If there are multiple solutions that are better than $\mathbf{x}$, there could be multiple choices for both the number of mutated bits $l$ and the location of ``$l$ preferred  bits''.

\begin{example}\label{exam1}
For the OneMax problem,  $e(\mathbf x)$ equals to the amount of `0'-bits in $\mathbf{x}$. Denoting $e(\mathbf{x})=j$ and $e(\mathbf{y})=i$, we know $\mathbf{y}$ replaces $\mathbf x$ if and only if $j\ge i$. Then, to generate a candidate $\mathbf y$ replacing $\mathbf x$,  ``$l$ preferred bits'' can be confirmed as follows.
\begin{itemize}
  \item If $i=j$, ``$l$ preferred bits'' consist of  $l/2$ `1'-bits and  $l/2$ `0'-bits, where $l$ is an even number that is not greater than $\min\{2j,2(n-j)\}$.
  \item While $i<j$, ``$l$ preferred bits'' could be combinations of $j-i+k$ `0'-bits and $k$ `1'-bits ($l=j-i+2k$), where $0\le k\le\min\{i,n-j\}$. Here, $k$ is not greater than $i$, because $j-i+k$ could not be greater than $j$, the number of `0'-bits in $\mathbf{x}$. Meanwhile, $k$ does not exceed $n-j$, the number of `1'-bits in $\mathbf{x}$.
\end{itemize}

\end{example}

If an EA flips each bit with an identical probability, the probability of flipping $l$ bits are related to $l$ and independent of their locations. Denoting the probability of flipping $l$ bits by $P(l)$, we can confirm the connection between the transition probability $r_{i,j}$ and $P(l)$.

\subsubsection{Transition Probabilities for  OneMax}
As presented in Example \ref{exam1}, transition from level $j$ to level $i$ ($i<j$) results from flips of $j-i+k$ `0'-bits and $k$ `1'-bits. Then,
\begin{equation}\label{POne}
 {r}_{i,j}=\sum_{k=0}^{M}{C_{n-j}^{k}C_{j}^{k+\left( j-i \right)}}P(2k+j-i),
\end{equation}
where $M=\min \left\{ n-j,i \right\}$, $0\le i<j\le n$.

\subsubsection{Transition Probabilities for Deceptive}
According to definition of the Deceptive problem, we get the following map from $\mid\mathbf{x}\mid$ to $e(\mathbf{x})$.
\begin{equation}\label{DecMap}
\begin{matrix}
     & \mid\mathbf{x}\mid:  & 0   & 1   & \cdots & n-1 & n \\
     & e(\mathbf{x}): & 1   & 2   & \cdots & n   & 0 \\
  \end{matrix}
\end{equation}
Transition from level $j$ to level $i$ ($0\le i<j\le n$) is attributed to one of the following cases.
\begin{itemize}
  \item If $i\ge 1$, the amount of `1'-bits decreases from $j-1$ to $i-1$. This transition results from change of $j-i+k$ `1'-bits and $k$ `0'-bits, where
      $0\le k\le \min\{n-j+1,i-1\}$;
  \item if $i=0$, all of $n-j+1$ `0'-bits are flipped, and all of its `1'-bits keep unchanged.
\end{itemize}

Accordingly, we know
\begin{align}
& {r}_{i,j}=\left\{\begin{aligned}&\sum_{k=0}^{M}{C_{n-j+1}^{k}C_{j-1}^{k+\left( j-i \right)}}P(2k+j-i), && i\ge 1,\\
& P(n-j+1), && i=0, \end{aligned}
\right.\label{PDec}
\end{align}
where $M=\min \left\{ n-j+1,i-1 \right\}$.

\subsection{Performance Metrics}

We propose two metrics to evaluate the performance of EAs,
which are the expected approximation error (EAE) and the tail probability (TP) of EAs for $t$ consecutive iterations. The approximation error was considered in previous work~\cite{he2016analytic} but the tail probability is a new performance metric.
\begin{definition}
Let $\{\mathbf{x}_t,t=1,2\dots\}$ be the individual sequence of an individual-based EA. The expected approximation error (EAE) after $t$ consecutive iterations is
\begin{equation}
  e^{[t]}=\mathbb{E}[e(\mathbf x_t)]=\sum_{i=0}^{L}e_i\Pr\{e(\mathbf x_t)=e_i\}.
\end{equation}
\end{definition}
  EAE  is the fitness gap between a solution and the optimum. It measures solution quality after running $t$ generations.

\begin{definition} Given $i>0$,
the tail probability (TP) of  the approximation error   that $e(\mathbf{x}_t)$ is greater than or equal to $e_i$ is defined as
\begin{equation}
  p^{[t]}(e_i)=\Pr\{e(\mathbf x_t)\ge e_i\}.
\end{equation}
\end{definition}

TP  is the probability distribution of a found solution over non-optimal levels where $i >0$. The sum of TP is the probability of not finding the optimum.

Given two EAs $\mathcal{A}$ and $\mathcal{B}$, if  both EAE and TP of Algorithm $\mathcal{A}$ are smaller than those of Algorithm $\mathcal{B}$ for any iteration budget, we say Algorithm $\mathcal{A}$ outperforms Algorithm $\mathcal{B}$ on problem (\ref{OP1}).

\begin{definition}\label{Def2}
  Let $\mathcal{A}$ and $\mathcal{B}$ be two EAs applied to problem (\ref{OP1}).
  \begin{enumerate}
\item   Algorithm $\mathcal{A}$  {outperforms} $\mathcal{B}$, denoted by $\mathcal{A}\succsim \mathcal{B}$, if it holds that
        \begin{itemize}
          \item $e_{\mathcal A}^{[t]}- e_{\mathcal B}^{[t]} \le 0$, $\forall\,t>0$;
          \item $p_{\mathcal A}^{[t]}(e_i)-  p_{\mathcal B}^{[t]}(e_i)\le 0$, $\forall\,t>0$, $0<i<L$.
        \end{itemize}
\item Algorithm $\mathcal{A}$  {asymptotically outperforms} $\mathcal{B}$ on problem (\ref{OP1}), denoted by $\mathcal{A}\succsim^{a} \mathcal{B}$, if it holds that
        \begin{itemize}
          \item $ \lim_{t\to \infty } e_{\mathcal A}^{[t]}- e_{\mathcal B}^{[t]} \le 0$;
          \item $\lim_{t\to+\infty} p_{\mathcal A}^{[t]}(e_i)- p_{\mathcal B}^{[t]}(e_i) \le 0$.
        \end{itemize}
  \end{enumerate}
\end{definition}

The asymptotic outperformance is weaker than the outperformance.

\section{Comparison of Transition Probabilities of the Two EAs} \label{SecTrans}
In this section, we
compare  transition probabilities of  $(1+1)EA$  and $(1+1)EA_C$. According to the connection  between $r_{i,j}$ and $P(l)$, comparison of transition probabilities can be conducted by considering the probabilities of flipping ``$l$ preferred bits''.

\subsection{Probabilities of Flipping   Preferred Bits }

Denote probabilities of $(1+1)EA$ and $(1+1)EA_C$ of flipping ``$l$ preferred bits'' by $P_1(l,p_m)$ and $P_2(l,C_R,q_m)$, respectively. By (\ref{Mut1}), we know
\begin{align}
  P_1(l,p_m)=(p_m)^l(1-p_m)^{n-l}. \label{P1}
\end{align}

Since the mutation and the binomial crossover in Algorithm \ref{Alg3} are mutually independent, we can get the probability by considering the crossover first. When flipping  ``$l$ preferred bits'' by the $(1+1)EA_C$, there are $l+k$ ($0\le k\le n-l$) bits of $\mathbf{y}$ set as $v_i$ by (\ref{Cross}), the probability of which  is
\begin{equation*}
  P_{C}(l+k,C_R)=\frac{l+k}{n}(C_R)^{l+k-1}(1-C_R)^{n-l-k}.
\end{equation*}
If only  ``$l$ preferred bits'' are flipped, we know,
\begin{align}
     P_2(l,C_R,q_m)&=\sum_{k=0}^{n-l}C_{n-l}^{k}P_C(l+k,C_R)(q_m)^l\left(1-q_m\right)^k                                                                                                                            \nonumber\\
     &=\frac{1}{n}\left[l+(n-l)C_R-nq_mC_R\right](C_R)^{l-1}(q_m)^l\left(1-q_mC_R\right)^{n-l-1}.\label{P3}
  \end{align}

Note that   $(1+1)EA_C$ degrades to  $(1+1)EA$ when $C_R=1$, and  $(1+1)EA$ becomes the random search while $p_m=1$. Thus, we assume that $p_m$, $C_R$ and $q_m$ are located in $(0,1)$. For a fair comparison of transition probabilities, we consider the identical parameter setting
\begin{equation}\label{ParaSetting}
  p_m=C_Rq_m=p,\quad 0<p<1.
\end{equation}

Then, we know $q_m=p/C_R$, and equation (\ref{P3}) implies
\begin{equation}\label{P3'}
{P_2(l,C_R,p/C_R)=\frac{1}{n}\left[(n-l)+\frac{l-np}{C_R}\right]p^l(1-p)^{n-l-1}}.
\end{equation}
Subtracting (\ref{P1}) from (\ref{P3'}), we have
\begin{align}
 & {P_2(l,C_R,p/C_R)-P_1(l,p)}
 ={\left\{\frac{1}{n}\left[(n-l)+\frac{l-np}{C_R}\right]-(1-p)\right\}p^l(1-p)^{n-l-1}}\nonumber\\
=&{\left(\frac{1}{C_R}-1\right)\left(\frac{l}{n}-p\right)p^l(1-p)^{n-l-1}}.\label{DP1P3}
\end{align}
From the fact that $0<C_R<1$, we conclude that $P_2(l,C_R,p/C_R)$ is greater than $P_1(l,p)$ if and only if $l>np$. That is, the introduction of the binomial crossover in the $(1+1)EA$ leads to the enhancement of exploration ability of the $(1+1)EA_C$. We get the following theorem  for the case that $p\le \frac{1}{n}$.

\begin{theorem}\label{T2}
  While $0<p\le \frac{1}{n}$, it holds for all $1\le l\le n$ that
  $P_1(l,p)\le P_2(l,C_R,p/C_R).$
\end{theorem}
\begin{proof}
The result can be obtained directly from equation (\ref{DP1P3}) by setting $p\le\frac{1}{n}$.
\end{proof}

\subsection{Comparison of Transition Probabilities}

 Given transition matrices from two EAs, one transition matrix dominating another is defined   as follows.

\begin{definition}\label{Def1}
  Let $\mathcal{A}$ and $\mathcal{B}$ be two EAs with an identical initialization mechanism. $\mathbf{\tilde{A}}=(a_{i,j})$ and $\mathbf{\tilde{B}}=(b_{i,j})$ are the transition matrices of $\mathcal{A}$ and $\mathcal{B}$, respectively.
  It is said that $\mathbf{\tilde{A}}$ \textbf{dominates} $\mathbf{\tilde{B}}$, denoted by $\mathbf{\tilde{A}}\succeq\mathbf{\tilde{B}}$, if it holds that
        \begin{enumerate}
          \item $a_{i,j}\ge b_{i,j},\quad\forall \,0\le i<j\le L$;
          \item $a_{i,j}> b_{i,j},\quad\exists\, 0\le i<j\le L$.
        \end{enumerate}
\end{definition}

Denote the transition probabilities of   $(1+1)EA$ and $(1+1)EA_C$ by $p_{i,j}$ and $s_{i,j}$, respectively.  For the OneMax problem and Deceptive problem, we get the relation of transition dominance on the premise that $p_m=C_Rq_m=p\le \frac{1}{n}$.

\begin{theorem}\label{T4}
For  $(1+1)EA$ and  $(1+1)EA_C$, denote their transition matrices by $\mathbf{\tilde{P}}$ and $\mathbf{\tilde{S}}$, respectively. On the condition that $p_m=C_Rq_m=p\le \frac{1}{n}$,  it holds for problem (\ref{OP1}) that
\begin{equation}\label{temp23}
  \mathbf{\tilde{S}} \succeq \mathbf{\tilde{P}}.
\end{equation}
\end{theorem}

\begin{proof}
Denote the collection of all solutions at level $k$ by $\mathcal{S}(k)$, $k=0,1,\dots,n$. We prove the result by considering the transition probability
$$r_{i,j}=\Pr\{\mathbf{y}\in\mathcal{S}(i)\mid\mathbf{x}\in\mathcal{S}(j)\},\quad (i< j).$$

Since the function values of solutions are only related to the number of `1'-bits, the probability to generate a solution $\mathbf{y}\in\mathcal{S}(i)$ by performing mutation on $\mathbf{x}\in\mathcal{S}(j)$ depends on the Hamming distance
\begin{equation*}
  l=H(\mathbf{x},\mathbf{y}).
\end{equation*}
Given $\mathbf{x}\in\mathcal{S}_j$, $\mathcal{S}(i)$ can be partitioned as
$$\mathcal{S}(i)=\bigcup_{l=1}^{L}\mathcal{S}_l(i),$$
where $\mathcal{S}_l(i)=\{\mathbf{y}\in\mathcal{S}(i)\mid H(\mathbf{x},\mathbf{y})=l\}$, and
$L$ is a positive integer that is smaller than or equal to $n$.

Accordingly, the probability to transfer from level $j$ to $i$ is confirmed as
\begin{align*}
 r_{i,j}=\sum_{l=1}^{L}\Pr\{\mathbf{y}\in\mathcal{S}_l(i)\mid\mathbf{x}\in\mathcal{S}(j)\}=\sum_{l=1}^{L}\mid\mathcal{S}_l(i)\mid P(l),
\end{align*}
where $\mid\mathcal{S}_l(i)\mid$ is the size of $\mathcal{S}_l(i)$, $P(l)$ the probability to flip ``$l$ preferred bits''. Thus, we have
\begin{align}
& p_{i,j}=\sum_{l=1}^{L}\Pr\{\mathbf{y}\in\mathcal{S}_l(j)\mid\mathbf{x}\}=\sum_{l=1}^{L}\mid\mathcal{S}_l(j)\mid P_1(l,p),\label{r1}\\
&  s_{i,j}=\sum_{l=1}^{L}\Pr\{\mathbf{y}\in\mathcal{S}_l(j)\mid\mathbf{x}\}=\sum_{l=1}^{L}\mid\mathcal{S}_l(j)\\mid P_2(l,C_R,p/C_R).\label{r3}
\end{align}
Since $p\le 1/n$, Theorem \ref{T2} implies that $$ P_1(l,p)\le P_2(l,C_R,p/C_R),\quad\forall\,\,1\le l\le n.$$
Combining it with (\ref{r1}) and (\ref{r3}) we know
\begin{equation}\label{temp21}
p_{i,j}\le s_{i,j}, \quad\forall\,\, 0\le i<j\le n.\end{equation}
Then, we get the result by Definition \ref{Def2}.
\end{proof}

\begin{example}\textbf{[Comparison of transition probabilities for the OneMax problem]}
Let $p_m=C_Rq_m=p\le\frac{1}{n}$. By (\ref{POne}), we have
\begin{align}
&{p_{i,j}}=\sum_{k=0}^M{C_{n-j}^{k}C_{j}^{k+\left( j-i \right)}}P_1(2k+j-i,p),\label{TP1}\\
&{s_{i,j}}=\sum_{k=0}^M{C_{n-j}^{k}C_{j}^{k+\left( j-i \right)}}P_2(2k+j-i,C_R,p/C_R).\label{TP3}
\end{align}
where $M={\min \left\{ n-j,i \right\}}$. Since $p\le 1/n$, Theorem \ref{T2} implies that $$ P_1(2k+j-i,p)\le P_2(2k+j-i,C_R,p/C_R),$$
and by (\ref{TP1}) and (\ref{TP3}) we have
$$p_{i,j}\le s_{i,j}, \quad\forall\,\, 0\le i<j\le n.$$

\end{example}

\begin{example}\label{exam3}
\textbf{[Comparison of transition probabilities for the Deceptive problem]}
Let $p_m=C_Rq_m=p\le\frac{1}{n}$. Equation (\ref{PDec}) implies that
\begin{equation}\label{TP11}
{p_{i,j}}=\left\{\begin{aligned}&\sum_{k=0}^{M}{C_{n-j+1}^{k}C_{j-1}^{k+\left( j-i \right)}}P_1(2k+j-i,p), && i>0,\\
& P_1(n-j+1,p), && i=0, \end{aligned}
\right.
\end{equation}

\begin{equation}\label{TP33}
{s_{i,j}}=\left\{\begin{aligned}& \sum_{k=0}^{M}{C_{n-j+1}^{k}C_{j-1}^{k+\left( j-i \right)}}P_2(2k+j-i,C_R,\frac{p}{C_R}), && i>0,\\
& P_2(n-j+1,C_R,p/C_R), && i=0, \end{aligned}
\right.
\end{equation}
where $M=\min \left\{ n-j+1,i-1 \right\}$. Similar to analysis of the Example 2, we know when $p\le 1/n$,
$$p_{i,j}\le s_{i,j}.$$
\end{example}

Nevertheless,  if $p>\frac{1}{n}$, we cannot get  Theorems \ref{T2}. Since
the differences among $p_{i,j}$ and $q_{i,j}$ depends on the characteristics of problem (\ref{OP1}),  Theorem \ref{T4} does not hold,  too.

\section{Analysis of Asymptotic Performance}\label{SecAsy}

Since the transition matrix of  $(1+1)EA_C$ dominates that of $(1+1)EA$, this section proves that  $(1+1)EA_C$ asymptotically outperforms $(1+1)EA$  using the average convergence rate~\cite{He2016,CHEN2021200}.

\begin{definition}
  The average convergence rate (ACR) of an EA for $t$ generation is
  \begin{equation}
    R_{EA}(t)=1-\left({e^{[t]}}/{e^{[0]}}\right)^{1/t}.
  \end{equation}
\end{definition}

The following lemma presents the asymptotic characteristics of the ACR, by which
we get the result on the asymptotic performance of EAs.
\begin{lemma}\label{Pro6}
\textbf{\citep[Theorem 1]{He2016}}
  Let $\mathbf{R}$ be the transition submatrix associated with a convergent EA. Under random initialization (i.e., the EA may start at any initial state with a positive probability), it holds
  \begin{equation}
    \lim_{t\to +\infty}R_{EA}(t)=1-\rho(\mathbf{R}),
  \end{equation}
  where $\rho(\mathbf{R})$ is the spectral radius of $\mathbf{R}$.
\end{lemma}

\begin{lemma}\label{Pro7}
If $\mathbf{\tilde{A}}\succeq\mathbf{\tilde{B}}$, there exists $T>0$ such that
\begin{enumerate}
  \item $e_{\mathcal{A}}^{[t]}\le e_{\mathcal{B}}^{[t]}$, $\forall\,t>T$;
  \item $p_{\mathcal{A}}^{[t]}(e_i)\le p_{\mathcal{B}}^{[t]}(e_i)$, $\forall\,t>T, \, 1\le i\le L$.
\end{enumerate}
\end{lemma}
\begin{proof}
By Lemma \ref{Pro6}, we know $\forall\,\epsilon>0$, there exists $T>0$ such that
\begin{equation}\label{temp19}
  e^{[0]}\left(\rho(\mathbf{R})-\epsilon\right)^t<e^{[t]} < e^{[0]}\left(\rho(\mathbf{R})+\epsilon\right)^t,\quad t>T.
\end{equation}
From the fact that the transition submatrix $\mathbf{R}$ of an RSH is upper triangular, we conclude
\begin{equation}\label{temp22}\rho(\mathbf{R})=\max\{r_{1,1},\dots, r_{L,L}\}.\end{equation}
Denote
$$\mathbf{\tilde{A}}=(a_{i,j})=\left(
               \begin{array}{cc}
                 1 & \mathbf{a}_0 \\
                 \mathbf{0} & \mathbf{A} \\
               \end{array}
             \right),\quad \mathbf{\tilde{B}}=(b_{i,j})=\left(
               \begin{array}{cc}
                 1 & \mathbf{b}_0 \\
                 \mathbf{0} & \mathbf{B} \\
               \end{array}
             \right).$$
While $\mathbf{\tilde{A}}\succeq\mathbf{\tilde{B}}$, it holds
$$ a_{j,j}=1-\sum_{i=0}^{j-1}a_{i,j}< 1-\sum_{i=0}^{j-1}b_{i,j}=b_{j,j},\,\,1\le j \le L.$$
Then, equation (\ref{temp22}) implies that
\begin{equation*}
  \rho(\mathbf{A})< \rho (\mathbf{B}).
\end{equation*}
Applying it to (\ref{temp19}) for $\epsilon <\frac{1}{2}(\rho(\mathbf{B})-\rho(\mathbf{A}))$, we have
\begin{equation}\label{temp20}
  e_{\mathcal{A}}^{[t]} < e^{[0]}\left(\rho(\mathbf{A})+\epsilon\right)^t<e^{[0]}\left(\rho(\mathbf{B})-\epsilon\right)^t<e_{\mathcal{B}}^{[t]},
\end{equation}
which proves the first conclusion.

Noting that the tail probability $p^{[t]}(e_i)$ can be taken as the expected approximation error of an optimization problem with error vector
$$\mathbf{e}=(\underbrace{0,\dots,0}_{i},1,\dots,1)',$$ by (\ref{temp20}) we have
$$p_{\mathcal{A}}^{[t]}(e_i)\le p_{\mathcal{B}}^{[t]}(e_i),\quad \forall\,t>T, \, 1\le i\le L.$$
The second conclusion is proven.
\end{proof}

Definition \ref{Def2} and Proposition \ref{Pro7} imply that dominance of transition matrix lead to the asymptotic outperformance. Then, we get the following theorem for comparing  the asymptotic performance  of  $(1+1)EA$ and $(1+1)EA_C$.
\begin{theorem}\label{T5}
If $C_R=C_Rq_m=p\le \frac{1}{n}$, the $(1+1)EA_C$  {asymptotically outperforms} the $(1+1)EA$ on problem (\ref{OP1}).
\end{theorem}
\begin{proof}
The proof can be completed by applying Theorem \ref{T4} and Lemma \ref{Pro7}.
\end{proof}

On condition that $C_R=C_Rq_m=p\le \frac{1}{n}$, Theorem \ref{T5} indicates that after sufficiently many number of iterations, the $(1+1)EA_C$ can performs better on problem (\ref{OP1}) than the $(1+1)EA$.

A further question is whether the $(1+1)EA_C$ outperforms the $(1+1)EA$ for $ t < +\infty$. We answer the question in next sections.

\section{Comparison of the Two EAs on OneMax}\label{SecExpi}

In this section, we show that the outperformance introduced by  binomial crossover can be obtained for the unimodel OneMax problem based  on the following lemma~\citep{WANG2021}.

\begin{lemma}\label{Th_ES}\textbf{\citep[Theorem 3]{WANG2021}}
Let
\begin{align*}
   &\mathbf{\tilde{e}}=(e_0,e_1,\dots,e_L)', \quad  \mathbf{\tilde{v}}=(v_0,v_1,\dots,v_L)',
\end{align*}
 where $0\le e_{i-1}\le e_{i}, i=1,\dots,L$, $v_i>0,i=0,1,\dots,L$. If transition matrices $\mathbf{\tilde{R}}$ and $\mathbf{\tilde{S}}$ satisfy
\begin{align}
\label{conC1}
&s_{j,j}  \ge  r_{j,j}, &&\forall\,\, j,
\\
\label{conC2}
& \sum^{i-1}_{l=0} (r_{l,j}-s_{l,j}) \ge 0 , &&\forall\,\,i<j,
\\
\label{conC3}
&  \sum^{i}_{l=0} ( s_{l,j-1}- s_{l,j})\ge 0 , &&\forall\,\,i<j-1,
\end{align}
 it holds
  $$\mathbf{\tilde{e}}'\mathbf{\tilde{R}}^t\mathbf{\tilde{v}}\le \mathbf{\tilde{e}}'\mathbf{\tilde{S}}^t\mathbf{\tilde{v}}.$$
\end{lemma}

For the EAs investigated in this study, conditions (\ref{conC1})-(\ref{conC3}) are satisfied thanks to the monotonicity of transition probabilities.
\begin{lemma}\label{L2}
When $p\le 1/n$ ($n\ge 3$), $P_1(l,p)$ and $P_2(l,C_R,p/C_R)$ are monotonously decreasing in $l$.
\end{lemma}
\begin{proof}
When $p\le 1/n$, equations (\ref{P1}) and (\ref{P3}) imply that
\begin{align}
 &   \frac{P_1(l+1,p)}{P_1(l,p)}=\frac{p}{1-p} \le \frac{1}{n-1},\label{temp4}\\
 & \frac{P_2(l+1,C_R,p/C_R)}{P_2(l,C_R,p/C_R)}=\frac{(l+1)(1-C_R)+nC_R(1-p/C_R)}{l(1-C_R)+nC_R(1-p/C_R)}\frac{p}{1-p}\le \frac{l+1}{l}\frac{p}{1-p}\le \frac{l+1}{l}\frac{1}{n-1},
\end{align}
all of which are not greater than $1$ when $n\ge 3$.
Thus, $P_1(l,p)$ and $P_2(l,C_R,p/C_R)$ are monotonously decreasing in $l$.
\end{proof}

\begin{lemma}\label{L3}
For the OneMax problem, $p_{i,j}$ and $s_{i,j}$ are monotonously decreasing in $j$.
 \end{lemma}
\begin{proof}
We validate the  monotonicity of $p_{i,j}$ for the $(1+1)EA$, and that of $q_{i,j}$ and $s_{i,j}$ can be confirmed in a similar way.

Let $0\le i< j< n$. By (\ref{TP1}) we know
\begin{align}
 & p_{i,j+1}=\sum_{k=0}^M{C_{n-j-1}^{k}C_{j+1}^{i-k}}P_1(2k+j+1-i,p),\label{temp1}\\
 &  p_{i,j}=   \sum_{k=0}^M{C_{n-j}^{k}C_{j}^{i-k}}P_1(2k+j-i,p),\label{temp2}
\end{align}
where $M={\min \left\{ n-j-1,i \right\}}$. Moreover, (\ref{temp4}) implies that
\begin{align*}
&  \frac{C_{j+1}^{i-k}P_1(2k+j+1-i,p)}{C_{j}^{i-k}P_1(2k+j-i,p)}=\frac{j+1}{(j+1)-(i-k)}\frac{p}{1-p}\le\frac{j+1}{2}\frac{1}{n-1}<1,
\end{align*}
and we know
\begin{equation}\label{temp3}
  {C_{j+1}^{i-k}P_1(2k+j+1-i,p)}<{C_{j}^{i-k}P_1(2k+j-i,p)}.
\end{equation}
Note that
\begin{align}\label{temp14}
& \min \left\{ n-j-1,i \right\}\ge \min \left\{ n-j,i \right\},\quad C_{n-j-1}^{k}<C_{n-j}^{k}.
\end{align}
From (\ref{temp1}), (\ref{temp2}), (\ref{temp3}) and (\ref{temp14}) we conclude that $$p_{i,j+1}<p_{i,j},\quad 0\le i< j< n.$$
Similarly, we can  validate that
$$ s_{i,j+1}<s_{i,j},\quad 0\le i< j< n.$$

In conclusion, $p_{i,j}$ and $s_{i,j}$ are monotonously decreasing in $j$.
\end{proof}

\begin{theorem}\label{T6}
On condition that $p_m=C_Rq_m=p\le \frac{1}{n}$, it holds for the OneMax problem that
\begin{equation*}
(1+1)EA_C \succsim (1+1)EA.
\end{equation*}
\end{theorem}
\begin{proof}
Given the initial distribution $\mathbf{\tilde{q}}^{[0]}$ and transition matrix $\mathbf{\tilde{R}}$, the level distribution at iteration $t$ is confirmed by
\begin{equation}\label{Dis_t}
\mathbf{\tilde{q}}^{[t]}=\mathbf{\tilde{R}}^t\mathbf{\tilde{q}}^{[0]}.
\end{equation}
Denote
\begin{align*}
   &\mathbf{\tilde{e}}=(e_0,e_1,\dots,e_L)', \quad  \mathbf{\tilde{o}}_i=(\underbrace{0,\dots,0}_{i},1,\dots,1)'.
\end{align*}
By premultiplying (\ref{Dis_t}) with $\mathbf{\tilde{e}}$ and $\mathbf{\tilde{o}}_i$, respectively, we get
\begin{align}
& e^{[t]}=\mathbf{\tilde{e}}'\mathbf{\tilde{R}}^t\mathbf{\tilde{q}^{[0]}},\label{EAE}\\
& p^{[t]}(e_i)=\Pr\{e(\mathbf{x}_t)\}\ge e_i\}=\mathbf{\tilde{o}}_i'\mathbf{\tilde{R}}^t\mathbf{\tilde{q}^{[0]}}.\label{TP}
\end{align}

Meanwhile, By Theorem \ref{T4} we have
\begin{align}
 & q_{j,j}\le s_{j,j} \le p_{j,j},\label{temp5}\\
 & \sum^{i-1}_{l=0} (q_{l,j}-s_{l,j}) \ge 0 ,\quad \sum^{i-1}_{l=0} (s_{l,j}-p_{l,j}) \ge 0 , \quad \forall\,\,i<j,\label{temp6}
\end{align}
and Lemma \ref{L3} implies
\begin{align}\label{temp7}
  \sum^{i}_{l=0} ( s_{l,j-1}- s_{l,j})\ge 0 , \quad \sum^{i}_{l=0} ( p_{l,j-1}- p_{l,j})\ge 0 &&\forall\,\,i<j-1.
\end{align}
Then, (\ref{temp5}), (\ref{temp6}) and (\ref{temp7}) validate satisfaction of conditions (\ref{conC1})-(\ref{conC3}), and by Lemma \ref{Th_ES} we know
\begin{align*}
 &  \mathbf{\tilde{e}}'\mathbf{\tilde{S}}^t\mathbf{\tilde{q}^{[0]}}\le \mathbf{\tilde{e}}'\mathbf{\tilde{P}}^t\mathbf{\tilde{q}^{[0]}}, &&\forall t>0;\\
 & \mathbf{\tilde{o}}_i'\mathbf{\tilde{S}}^t\mathbf{\tilde{q}^{[0]}}\le \mathbf{\tilde{o}}_i'\mathbf{\tilde{P}}^t\mathbf{\tilde{q}^{[0]}}, &&\forall t>0,\,1\le i< n.
\end{align*}
Then, we get the conclusion by Definition \ref{Def2}.
\end{proof}

The above theorem demonstrates that the dominance of transition matrices introduced by the binomial crossover operator leads to the outperformance of $(1+1)EA_C$ on the unimodal problem OneMax.

\section{Comparison of the two EAs on Deceptive and Adaptive Parameter Strategy}\label{SecExpr}
In this section, we show  that the outperformance of  $(1+1)EA_C$ over $(1+1)EA$ may not always hold on Deceptive, and then, propose an adaptive strategy of parameter setting arising from the theoretical analysis.
\subsection{A counterexample for inconsistency between the transition dominance and the algorithm outperformance}

For the Deceptive problem, we present a counterexample to show even if the transition matrix of an EA dominates another EA, we cannot draw that the former EA outperform the later.

\begin{example}\label{Exam}
We construct two artificial Markov chains as the models of two EAs.
Let $EA_\mathcal{R}$ and $EA_\mathcal{S}$ be two EAs staring with an identical initial distribution
$$\mathbf{p}^{[0]}=\left(\frac{1}{n},\frac{1}{n},\dots,\frac{1}{n}\right)^t,$$
and the respective transition matrices are
  \begin{align*}
    \mathbf{\tilde{R}}=\begin{pmatrix}
                         1 & \frac{1}{n^3} &  \frac{2}{n^3} & \dots  &\frac{n}{n^3} \\
                           & 1-\frac{1}{n^3} & \frac{1}{n^2} &  &  \\
                           &  & 1-\frac{1}{n^2}-\frac{2}{n^3} & \ddots  &  \\
                           &  &         & \ddots  &   \frac{n-1}{n^2}\\
                           &  &         &   &  1-\frac{1}{n}\\
                       \end{pmatrix}
   \end{align*}
   and
   \begin{align*}
    \mathbf{\tilde{S}}=\begin{pmatrix}
                         1 & \frac{2}{n^3} &  \frac{4}{n^3} &\dots  &\frac{2n}{n^3} \\
                           & 1-\frac{2}{n^3} & \frac{1}{n^2}+\frac{1}{2n} &  &  \\
                           &  & 1-\frac{n^2+2n+8}{2n^3} &\ddots  &  \\
                           &  &         &\ddots  &   \frac{n-1}{n^2}+\frac{n-1}{2n}\\
                           &  &         &   &  1-\frac{n^2+n+2}{2n^2}\\
                       \end{pmatrix}.\\
  \end{align*}
\end{example}

Obviously, it holds $\mathbf{\tilde{S}}\succeq \mathbf{\tilde{R}}$.
Through computer simulation, we get the  curve  of  EAE difference of the two EAs  in Figure~\ref{Test}(a) and the curve of TPs  difference  of the two EAs  in Figure~\ref{Test}(b). From Figure~\ref{Test}(b), it is clear that  $EA_\mathcal{R}$ does not always outperform $EA_\mathcal{S}$ because the difference of TPs is negative at the early stage of the iteration process but later positive.

\begin{figure}[!t]
\centering
\begin{minipage}[c]{0.48\textwidth}
\centering
\includegraphics[width=2.5in]{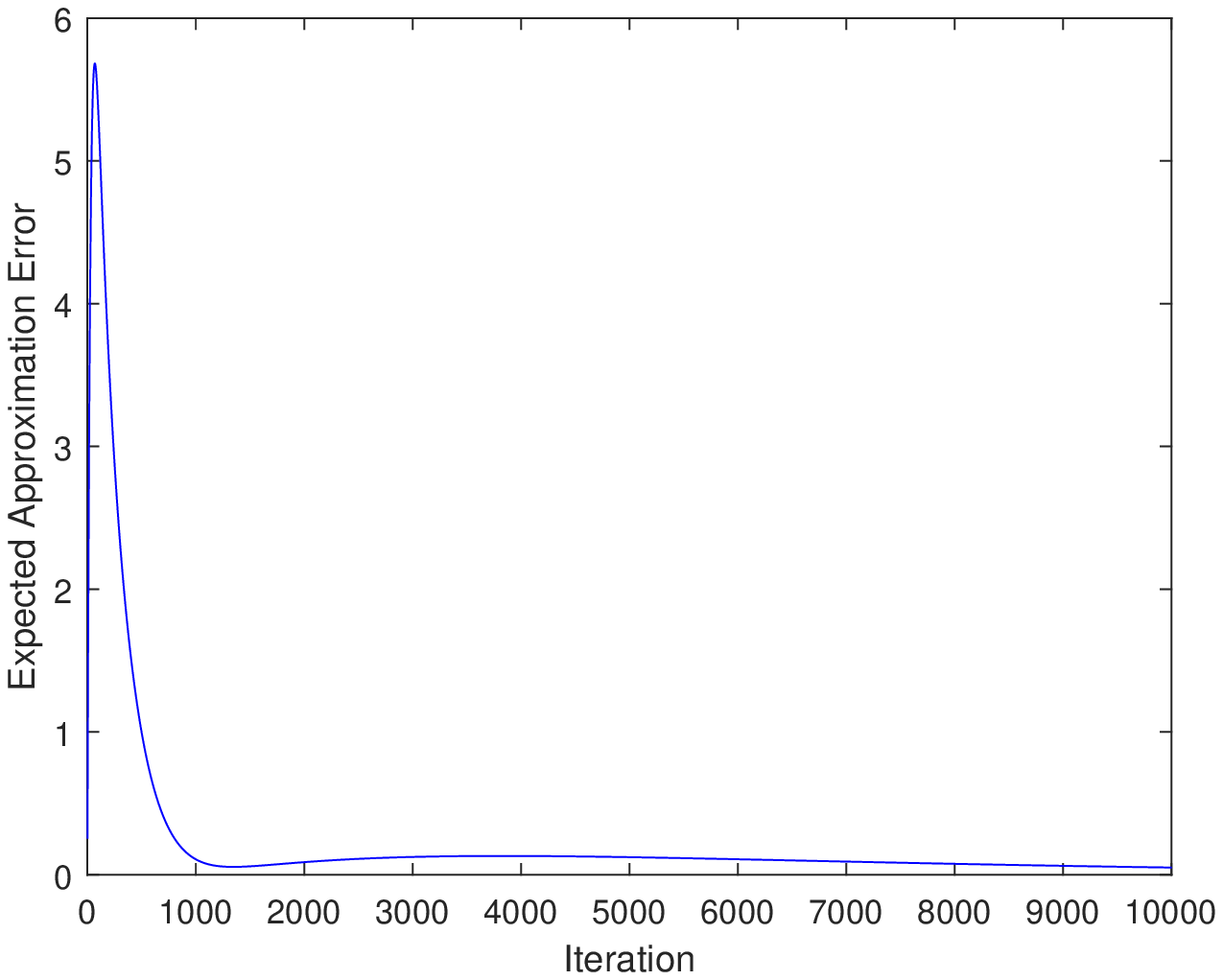}
\end{minipage}
\hspace{0.02\textwidth}
\begin{minipage}[c]{0.48\textwidth}
\centering
\includegraphics[width=2.5in]{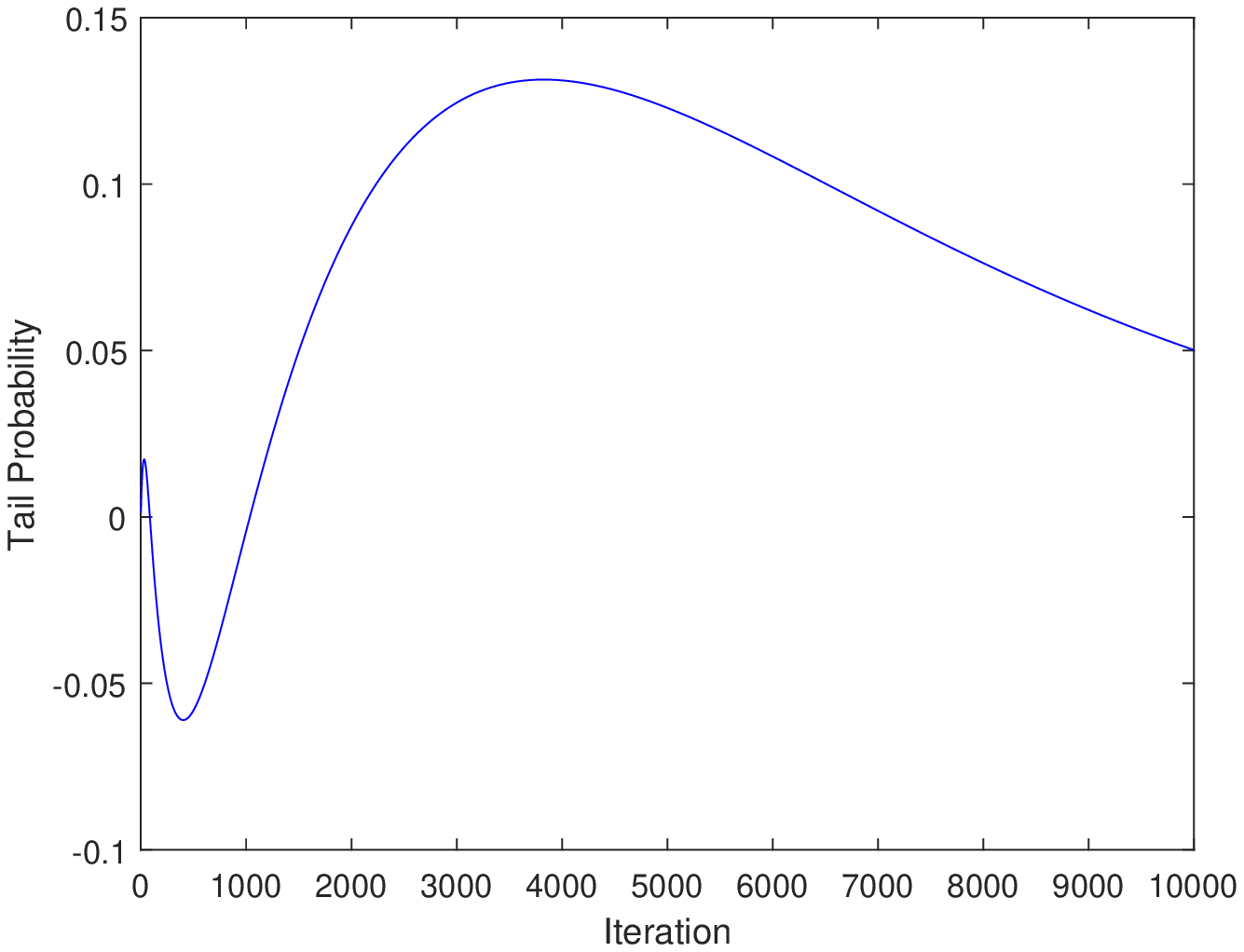}
\end{minipage}
\caption{Simulation results of the difference of EAEs and TPs for the artificial case. }
\label{Test}
\end{figure}

\subsection{Numerical comparison for the Two EAs on Deceptive}

Now we turn  to discuss $(1+1)EA$ and $(1+1)EA_C$ on Deceptive. We demonstrate $(1+1)EA_C$ may not outperform $(1+1)EA$ over all generations although the transition matrix of $(1+1)EA_C$ dominates that of $(1+1)EA$.

\begin{example}
In  $(1+1)EA$ and  $(1+1)EA_C$, set  $$ p_m=C_Rq_m=\frac{1}{n}.$$
For $(1+1)EA_C$, let $q_m=\frac{1}{2}$, $C_R=\frac{2}{n}$. The numerical simulation results of EAEs and TPS for 5000 independent runs are depicted in Figure \ref{fig2}.
It is shown that when $n\ge 9$, both EAEs and TPS of $(1+1)EA$ could be smaller than those of  $(1+1)EA_C$. This indicates that the dominance of transition matrix does not always guarantee the outperformance of the corresponding algorithm.
\end{example}

\begin{figure}[!t]
\centering
\begin{minipage}[c]{0.48\textwidth}
\centering
\includegraphics[width=2.5in]{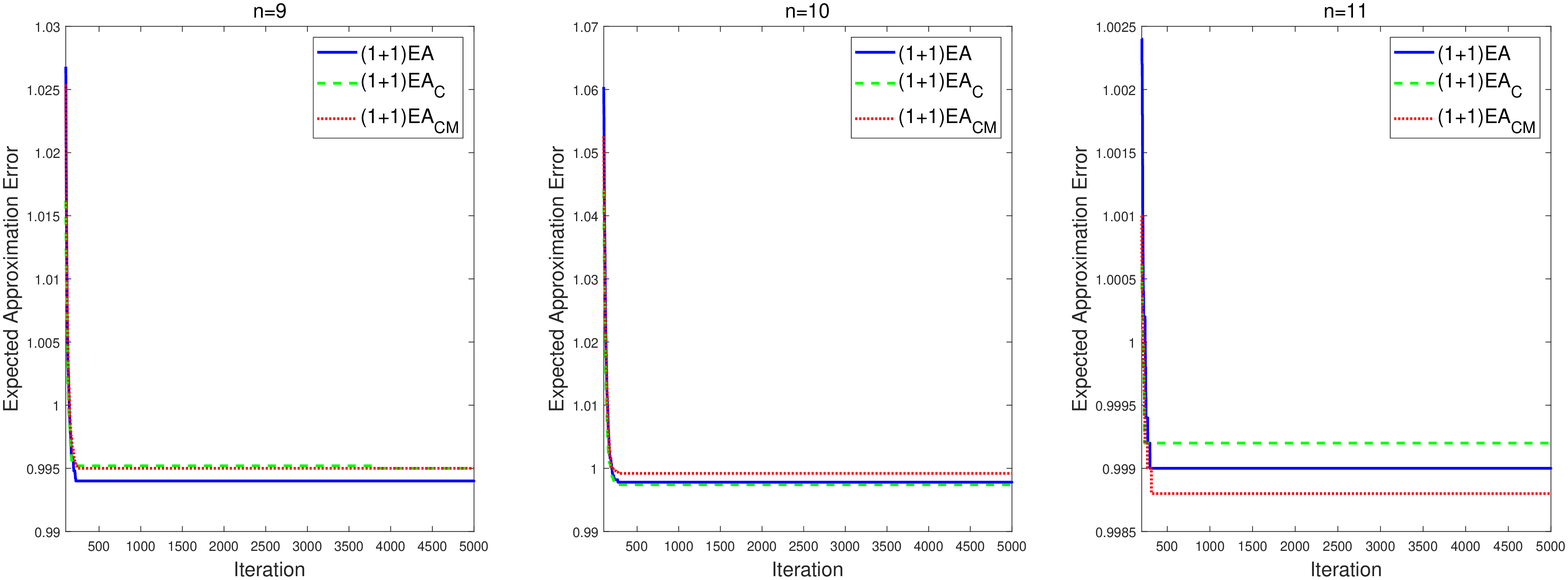}
\end{minipage}
\hspace{0.02\textwidth}
\begin{minipage}[c]{0.48\textwidth}
\centering
\includegraphics[width=2.5in]{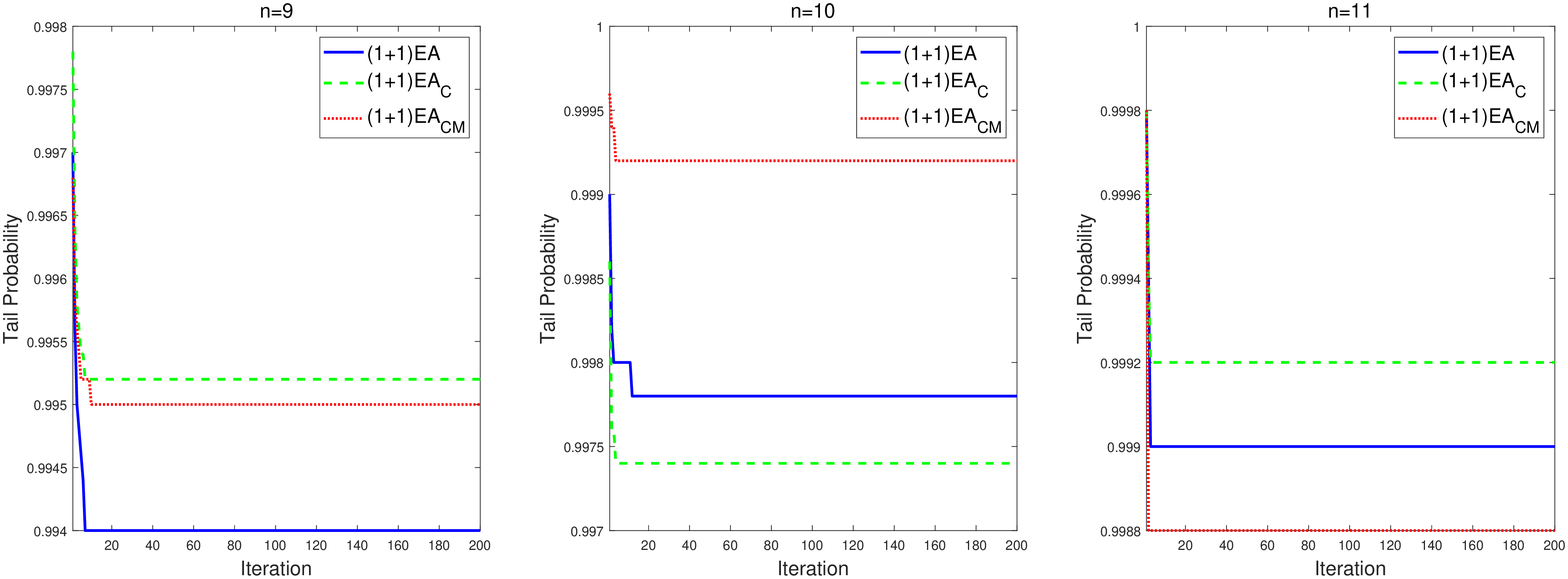}
\end{minipage}
\caption{Numerical comparison on expected approximation errors (EAEs) and tail probabilities (TPs) of $(1+1)EA$ and $(1+1)EA_C$ applied to the Deceptive problem, where $n$ refers to the problem dimension. }
\label{fig2}
\end{figure}

With $p_m=C_Rq_m=p\le \frac{1}{n}$, although the binomial crossover leads to transition dominance of the $(1+1)EA_C$ over the $(1+1)EA$, the enhancement of exploitation plays a governing role in the iteration process. Thus, the imbalance of exploration and exploitation leads to a poor performance of $(1+1)EA_C$ on some stage of the iteration process.

\subsection{Comparisons on the Probabilities to Transfer from Non-optimal Statuses to the Optimal Status}
As shown in the previous two counterexamples, the outperformance of $(1+1)EA){MC}$ cannot be drawn from dominance of transition matrices. To  enhance the performance of an EA, adaptive parameter settings should be incorporated to  improve the exploration ability of an  EA on Deceptive.

Local exploitation could result in convergence to the local optimal solution, global convergence of EAs on Deceptive  is principally attributed to the direct transition from level $j$ to level $0$, which is quantified by the transition probability $r_{0,j}$. Thus, we investigate the impact of  binomial crossover  on the transition probability $r_{0,j}$, and accordingly, arrive at the strategies for adaptive regulations of the mutation rate and the crossover rate.

In the following, we first compare $p_{0,j}$ and $s_{0,j}$ by investigating their monotonicity, and then, get reasonable adaptive strategies that improve the performance of EAs on the Deceptive problem.

Substituting (\ref{P1}) and (\ref{P3}) to (\ref{TP11}) and (\ref{TP33}), respectively, we have
\begin{align}
 &  p_{0,j}  =P_1(n-j+1,p_m)=(p_m)^{n-j+1}(1-p_m)^{j-1}, \label{TP111}\\
 &  s_{0,j}  =P_{3}(n-j+1,C_R,q_m)\nonumber\\
 &  = \frac{1}{n}\left[(j-1)(1-C_R)+nC_R(1-q_m)\right]C_R^{n-j}(q_m)^{n-j+1}\left(1-q_mC_R\right)^{j-2}. \label{TP333}
\end{align}

We first investigate the maximum values of $p_{0,j}$ to get the ideal performance of $(1+1)EA$  on the Deceptive problem.

\begin{theorem}\label{Pro1}
While
  \begin{align}
 & p^{\star}_m=\frac{n-j+1}{n},\label{temp10}
\end{align}
$p_{0,j}$ gets its maximum values
\begin{equation}
 p_{0,j}^{max}=\left(\frac{n-j+1}{n}\right)^{n-j+1}\left(\frac{j-1}{n}\right)^{j-1}.
\end{equation}
\end{theorem}
\begin{proof}
By (\ref{TP111}), we know
\begin{align*}
{\frac{\partial}{\partial(p_m)}p_{0,j}=(n-j+1-np_m)p_m^{n-j}\left(1-p_m\right)^{j-2}}.
\end{align*}
While $  p_m=\frac{n-j+1}{n}$,
$p_{0,j}$ gets its maximum value
\begin{equation*}
     p_{0,j}^{max}=P_1(n-j+1,\frac{n-j+1}{n})=\left(\frac{n-j+1}{n}\right)^{n-j+1}\left(\frac{j-1}{n}\right)^{j-1}.
\end{equation*}
\end{proof}

Influence of the binomial crossover on $s_{0,j}$ is investigated  on condition that $p_m=q_m$. By regulating the crossover rate $C_R$, we can compare $p_{0,j}$ with the maximum value $s_{0,j}^{max}$ of $s_{0,j}$.

\begin{theorem}\label{Pro5}
On condition that $p_m=q_m$, the following results hold.
\begin{enumerate}
  \item $p_{0,1}= s^{max}_{0,1}$.
  \item If $q_m>\frac{n-1}{n}$, $p_{0,2}< s^{max}_{0,2}$; otherwise, $p_{0,2}= s^{max}_{0,2}$.
  \item $\forall\,\, j\in\{3,\dots,n-1\}$, $p_{0,j}\le s^{max}_{0,j}$ if $q_m>\frac{n-j}{n-1}$; otherwise, $s^{max}_{0,j}=p_{0,j}$.
  \item if $q_m>\frac{1}{n}$, $p_{0,n}< s^{max}_{0,n}$; otherwise, $s^{max}_{0,n}=p_{0,n}$.
\end{enumerate}
\end{theorem}
\begin{proof}
Note that $(1+1)EA_C$ degrades to $(1+1)EA$ when $C_R=1$. Then, if the maximum value $s_{0,j}^{max}$ of $s_{0,j}$ is obtained by setting $C_R=1$, we have $s_{0,j}^{max}=p_{0,j}$; otherwise, it holds $s_{0,j}^{max}>p_{0,j}$.
\begin{enumerate}
\item For the case that $\boldsymbol{j=1}$,  equation (\ref{TP333}) implies
\begin{equation*}
  {{s}_{0,1}}=q_{m}^{n}{{\left( C_R \right)}^{n-1}}.
\end{equation*}
Obviously, $s_{0,1}$ is monotonously increasing in $C_R$. It gets the maximum value while $C_R^{\star}=1$. Then, by (\ref{TP111}) we have $s^{max}_{0,1}= p_{0,1}$.

\item While $\boldsymbol{j=2}$, by (\ref{TP333}) we have
\begin{equation*}
  \frac{\partial s_{0,2}}{\partial C_R}=\frac{n-1}{n}q_{m}^{n-1}{{\left( C_R \right)}^{n-3}}\left( n-2\text{+}\left( 1-n{{q}_{m}} \right)C_R \right).
\end{equation*}

\begin{itemize}
\item If $0<{{q}_{m}}\le \frac{n-1}{n}$, ${s}_{0,2}$ is  monotonously increasing in $C_R$, and gets  its maximum value while $C_R^{\star}=1$. For this case,  we know $s^{max}_{0,2}= p_{0,2}$.
\item While $\frac{n-1}{n}<{q}_{m}<1$, ${s}_{0,2}$ gets  its maximum value $s_{0,2}^{max}$ by setting
\begin{equation} C_R^{\star}=\frac{n-2}{nq_m-1}.\end{equation}
Then, we have $s_{0,2}^{max}>p_{0,2}$.
\end{itemize}
\item For the case that $\boldsymbol{3\le j\le n-1}$,  we denote
\begin{equation*}
  {s}_{0,j}=\frac{n-j+1}{n}q_{m}^{n-j+1}I_1+\frac{\left( j-1 \right)\left( 1-{{q}_{m}} \right)}{n}q_{m}^{n-j+1}I_2,
\end{equation*}
where
\begin{align*}
   & I_1={{\left( C_R \right)}^{n-j}}{{\left( 1-{{q}_{m}}C_R \right)}^{j-1}}, \\
   & I_2={{\left( C_R \right)}^{n-j+1}}{{\left( 1-{{q}_{m}}C_R \right)}^{j-2}}.
\end{align*}
Then,
\begin{equation*}
\frac{\partial I_1}{\partial C_R}={{\left( C_R \right)}^{n-j-1}}{{\left( 1-{{q}_{m}}C_R \right)}^{j-2}}\left( n-j-\left( n-1 \right){{q}_{m}}C_R \right),
\end{equation*}
\begin{equation*}
  \frac{\partial I_2}{\partial C_R}={{\left( C_R \right)}^{n-j}}{{\left( 1-\frac{C_R}{n} \right)}^{j-3}}\left( n-j+1-\left( n-1 \right){{q}_{m}}C_R \right).
\end{equation*}
\begin{itemize}
  \item While $0<q_m\le\frac{n-j}{n-1}$, both $I_1$ and $I_2$ are monotonously increasing in $C_R$. For this case, $s_{0,j}$ gets its maximum value when $C_R^{\star}=1$, and we have $s^{max}_{0,j}= p_{0,j}$.
  \item If $\frac{n-j+1}{n-1}\le q_m\le 1$,  $I_1$ gets its maximum value when $C_R=\frac{n-j}{(n-1)q_m}$, and $I_2$ gets its maximum value when $C_R=\frac{n-j+1}{(n-1)q_m}$. Then, $s_{0,j}$ could get its maximum value $s_{0,j}^{max}$ at some
      \begin{equation}\label{temp8}
       F^{\star}_R\in \left(\frac{n-j}{(n-1)q_m},\frac{n-j+1}{(n-1)q_m}\right).\end{equation}
      Then, we kmow $s^{max}_{0,j}> p_{0,j}$.
  \item If $\frac{n-j}{n-1}< q_m< \frac{n-j+1}{n-1}$, $I_1$ gets its maximum value when $C_R=\frac{n-j}{(n-1)q_m}$,  and $I_2$ is monotonously increasing in $C_R$. Then, $s_{0,j}$ could get its maximum value $s_{0,j}^{max}$ at some
      \begin{equation}\label{temp15} F^{\star}_R\in \left(\frac{n-j}{(n-1)q_m},1\right].\end{equation}
      Then,
      $s_{0,j}^{max}>p_{0,j}$.
\end{itemize}

\item While {$\boldsymbol{j=n}$}, equation (\ref{TP333}) implies that
\begin{align*}
\frac{\partial s_{0,n}}{\partial C_R}=\left( n-1 \right){{\left( 1-{{q}_{m}}C_R \right)}^{n-3}}\left( 1-2{{q}_{m}}-\left( n-1-n{{q}_{m}} \right){{q}_{m}}C_R \right).
\end{align*}
Denoting
\begin{equation*}g\left(q_m, C_R \right)=1-2{{q}_{m}}-\left( n-1-n{{q}_{m}} \right){{q}_{m}}C_R,\end{equation*}
we can confirm the sign of ${\partial s_{0,n}}/{\partial C_R}$  by considering
$$\frac{\partial}{\partial C_R}g\left(q_m, C_R \right)=-\left( n-1-n{{q}_{m}} \right){{q}_{m}}.$$
\begin{itemize}
\item While $0<{{q}_{m}}\le \frac{n-1}{n}$, $g\left(q_m, C_R \right)$ is monotonously decreasing in $C_R$, and its minimum value is
    $$g\left(q_m, 1 \right)=\left( n{{q}_{m}}-1 \right)\left( {{q}_{m}}-1 \right).$$
    The maximum value of $g\left(q_m, C_R \right)$ is
    $$g\left(q_m, 0 \right)=1-2q_m.$$
    \begin{enumerate}
    \item If $0<{{q}_{m}}\le \frac{1}{n}$, we have
    $$g\left(q_m, C_R \right)\ge g\left(q_m, 1 \right)>0.$$
    Thus, $\frac{\partial s_{0,n}}{\partial C_R}\ge 0$, and ${{s}_{0,n}}$ is monotonously increasing in $C_R$. For this case, $s_{0,n}$ get its maximum value when $C_R^{\star}=1$, and we have $s_{0,n}^{max}=p_{0,n}$.
    \item    If $\frac{1}{n}<{{q}_{m}} \le \frac{1}{2}$, ${{s}_{0,n}}$ gets the maximum value $s_{0,n}^{max}$ when
    $$ C_R^{\star}=\frac{1-2q_m}{q_m(n-1-nq_m)}.$$
    Thus, $s_{0,n}^{max}>p_{0,n}$.
    \item If $\frac{1}{2}<{{q}_{m}} \le \frac{n-1}{n}$, $g\left(q_m, 0 \right)<0$, and then, ${{s}_{0,n}}$ is monotonously decreasing in $C_R$. Then, its maximum value is obtained by setting $C_R^{\star}=0$. Then, we know $s^{max}_{0,n}>p_{0,n}$.
\end{enumerate}
    \item While $ \frac{n-1}{n}< {{q}_{m}} \le 1$, $g\left(q_m, C_R \right)$ is monotonously increasing in $C_R$, and its maximum value
    $$g\left(q_m, 1 \right)=\left( n{{q}_{m}}-1 \right)\left( {{q}_{m}}-1 \right)<0.$$
    Then, ${{s}_{0,n}}$ is monotonously decreasing in $C_R$, and its maximum value is obtained by setting $C_R^{\star}=0$. Then, $s^{max}_{0,n}>p_{0,n}$.
\end{itemize}
In summary, $s^{max}_{0,n}>p_{0,n}$ while $q_m>\frac{1}{n}$; otherwise, $s^{max}_{0,n}=p_{0,n}$.
\end{enumerate}
\end{proof}

Theorems \ref{Pro1} and \ref{Pro5} presents the ``best'' settings to maximize the transition probabilities from non-optimal statues to the optimal level. Unfortunately, such results are  not available if the global optima solution is unknown.

\subsection{Parameter Adaptive Strategy to Enhance Exploration of EAs}
We proposes a parameter Adaptive Strategy based on Hamming distance. Since the level index $j$ is equal to the Hamming distance between $\mathbf{x}$ and $\mathbf{x}^*$, improvement of level index $j$ is in deed equal to reduction of the Hamming distance obtained by replacing $\mathbf{x}$ with $\mathbf{y}$. Then, while the local exploitation leads to transition from level $j$ to a non-optimal level $i$, the practically adaptive strategy of parameters can be obtained according to the Hamming distance between $\mathbf{x}$ and $\mathbf{y}$.

For the Deceptive problem, consider two solutions $\mathbf{x}$ and $\mathbf{y}$ such that $f(\mathbf{y})>f(\mathbf{x})$, and denote their levels by $j$ and $i$, respectively. When the $(1+1)EA$ is located at the solution $\mathbf{x}$, equation (\ref{temp10}) implies that the ``best'' setting of mutation rate is $p^{\star}_m=\frac{n-j+1}{n}$. While the promising solution $\mathbf y$ is generated, the level transfers from $j$ to $i$, and the ``best'' setting change to $p^{\star}_m=\frac{n-i+1}{n}$. Let $H(\mathbf{x},\mathbf{y})$ denote the Hamming distance between $\mathbf{x}$ and $\mathbf{y}$. Noting that $H(\mathbf{x},\mathbf{y})\ge j-i$, we know the difference of ``best'' parameter settings is bounded from above by $\frac{H(\mathbf{x},\mathbf{y})}{n}$. Accordingly, the mutation rate of $(1+1)EA$ can be updated to
\begin{equation}\label{temp16}
  p^{\prime}_m=p_m+\frac{H(\mathbf{x},\mathbf{y})}{n}.
\end{equation}

For the $(1+1)EA_C$,  the parameter $q_m$ is adapted  using the  strategy consistent to that of $p_m$ to focus on influence of $C_R$. That is,
\begin{equation}\label{temp13}
  q^{\prime}_m=q_m+\frac{H(\mathbf{x},\mathbf{y})}{n}.
\end{equation}
Since $s_{0,j}$ demonstrates different monotonicity for varied levels, one cannot get an identical strategy for adaptive setting of $C_R$. As a compromise, we would like to consider the case that $3\le j\le n-1$, which is obtained by random initialization with overwhelming probability.

According to the proof of Theorem \ref{Pro5}, we know $C_R$ should be set as great as possible for the case $q_m\in(0,\frac{n-j}{n-1}]$; while $q_m\in(\frac{n-j}{n-1}, 1]$,  $C_R^{\star}$ is located in intervals whose boundary values are $\frac{n-j}{(n-1)q_m}$ and $\frac{n-j+1}{(n-1)q_m}$, given by (\ref{temp8}) and (\ref{temp15}), respectively. Then, while $q_m$ is updated by (\ref{temp13}), the update strategy of $C_R$ can be confirmed to satisfying that
$$ F^{\prime}_Rq^{\prime}_m=C_Rq_m+\frac{H(\mathbf{x},\mathbf{y})}{n-1}.$$
Accordingly, the adaptive setting of $C_R$ could be

\begin{equation}\label{temp18}
 F^{\prime}_R=\left(C_Rq_m+\frac{H(\mathbf{x},\mathbf{y})}{n-1}\right)/q^{\prime}_m,
\end{equation}
where $q'_m$ is updated by (\ref{temp13}).

To demonstrate the promising function of the adaptive update strategy, we incorporate it to  $(1+1)EA$ and $(1+1)EA_C$ to get the adaptive variant of these algorithms, and test their performance on the 12-20 dimensional Deceptive problems.
Parameters of EAs are initialized by (\ref{ParaSetting}), and adapted according to (\ref{temp16}), (\ref{temp13}) and (\ref{temp18}), respectively.

Since the adaptive strategy decreases stability of performances to a large extent, numerical simulation of the tail probability is implemented by 10,000 independent runs. To investigate the sensitivity of the adaptive strategy on initial values of $q_m$, the mutation rate $q_m$ in $(1+1)EA_C$ is initialized with values $\frac{1}{\sqrt{n}}$, $\frac{3}{2\sqrt{n}}$ and $\frac{2}{\sqrt{n}}$, where the three variants are denoted by  $(1+1)EA_C^1$, $(1+1)EA_C^2$ and $(1+1)EA_C^3$, respectively.

The converging curves of averaged TPs are illustrated in Figure \ref{CompDecAdaptive}. Compared to the EAs with fixed parameters during the evolution process, the performance of adaptive EAs has significantly been improved. Meanwhile, the outperformance of  $(1+1)EA_C$ introduced by the binomial crossover is greatly enhanced by the adaptive strategies, too.

\begin{figure*}[ht]
  \centering
  \includegraphics[width=7in]{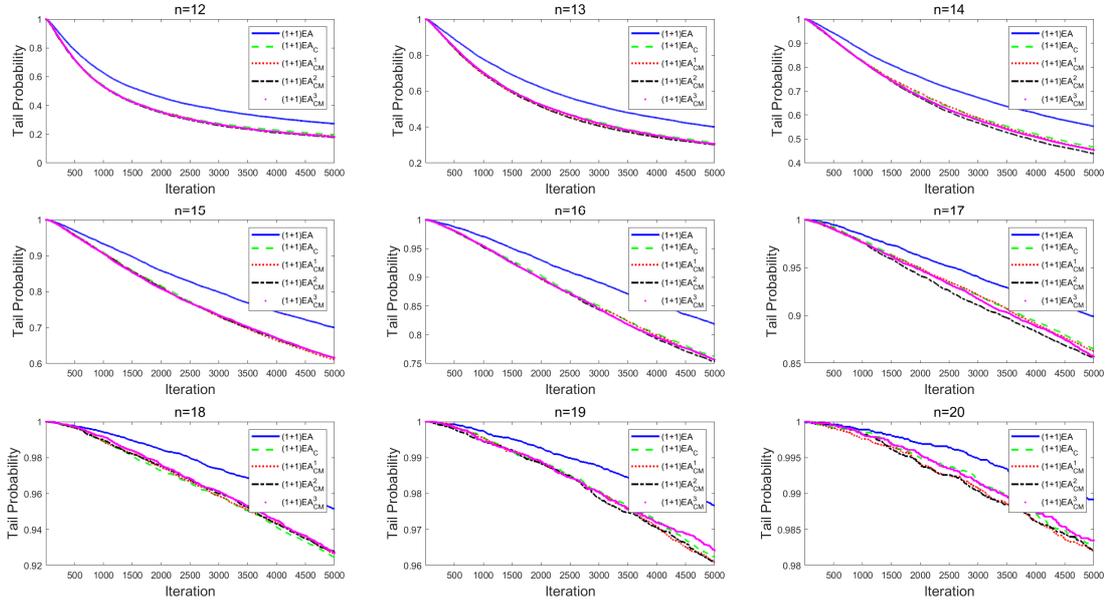}
  \caption{Numerical comparison on tail probabilities (TPs) of adaptive $(1+1)EA$, $(1+1)EA_C$ and $(1+1)EA_C$ applied to the Deceptive problem, where $n$ is the problem dimension. The $(1+1)EA_C^1$, $(1+1)EA_C^2$ and $(1+1)EA_C^3$ are three variants of $(1+1)EA_C$ with $q_m$ initialized as $\frac{1}{\sqrt n}$, $\frac{3}{2\sqrt n}$ and $\frac{2}{\sqrt n}$, respectively.}\label{CompDecAdaptive}
\end{figure*}

\section{Conclusions}\label{SecCon}
Under the framework of fixed budget analysis, we conduct a pioneering analysis of the influence of binomial crossover on the approximation error of EAs. The performance of an EA after running finite generations is measured by two metrics: the expected value of the approximation error and the error tail probability.
Using the two metrics, we make a case study of  comparing the approximation error of $(1+1)EA$  and $(1+1)EA_C$ with binomial crossover.

Starting from the comparison on probability of flipping ``{\it l preferred bits}'', it is proven that under proper conditions, incorporation of  binomial crossover leads to dominance of transition probabilities, that is, the  probability of transferring to any promising status is improved. Accordingly, the asymptotic performance of $(1+1)EA_C$ is superior to that of $(1+1)EA$.

 It is found that the dominance of transition probability guarantees that $(1+1)EA_C$ outperforms $(1+1)EA$ on OneMax in terms of both expected approximation error and tail probability. However, this dominance does leads to the outperformance on Deceptive.  This means that using binomial crossover may reduce the approximation error on some problems   but   not  on other problems.

For Deceptive, an adaptive strategy of parameter setting is proposed based on the monotonicity analysis of transition probabilities. Numerical simulations demonstrate that it can significantly improve the  exploration ability of $(1+1)EA_C$ and $(1+1)EA$, and superiority of binomial crossover is further strengthened by the adaptive strategy.   Thus, a problem-specific  adaptive strategy is helpful for  improving the performance of EAs.

Our future work will focus on a further study  for adaptive setting of crossover rate in population-based EAs on more complex problems, as well as development of adaptive EAs improved by introduction of  binomial crossover.

\section*{Acknowledgements}
This research was supported  by the Fundamental Research Funds for the Central
Universities (WUT: 2020IB006).


\bibliography{mybibfile}

\end{document}